\documentclass[letterpaper, 10 pt, conference]{ieeeconf} 
\IEEEoverridecommandlockouts                              
\usepackage{multirow}
\usepackage{amsmath}

\usepackage{verbatim}
\usepackage{amssymb}
\usepackage{mathtools}
\usepackage{commath}
\usepackage{nth}
\usepackage{cite}
\usepackage{cleveref}
\usepackage{caption,url}
\usepackage{subcaption}
\usepackage{xcolor}
\usepackage{dblfloatfix} 
\usepackage{float} 

\usepackage{soul}

\usepackage{version}

\excludeversion{old_text}

\usepackage[vlined,ruled]{algorithm2e}

\usepackage {tikz}
\usetikzlibrary{shapes ,arrows , positioning }
\def\prob{\mathbb{P}}
\def\expt{\mathbb{E}}
\def\real{\mathbb{R}}

\def\natural{\mathbb{N}}

\usepackage{amsmath,graphicx,epsfig,color,amsfonts}

\graphicspath{{./fig/}}

\def\prob{\mathbb{P}}
\def\expt{\mathbb{E}}
\def\real{\mathbb{R}}

\def\natural{\mathbb{N}}
\newcommand{\until}[1]{\{1,\dots, #1\}}

\newcommand{\supscr}[2]{#1^{\textup{#2}}}

\newcommand\oprocendsymbol{\hbox{$\square$}}
\newcommand\oprocend{\relax\ifmmode\else\unskip\hfill\fi\oprocendsymbol}

\newcommand\bit[1]{\textit{\textbf{#1}}}

\def \mc {\mathcal}

\newtheorem{theorem}{Theorem}

\newtheorem{lemma}[theorem]{Lemma}

\renewcommand{\theenumi}{(\roman{enumi}}

\usepackage{caption}
\renewcommand{\baselinestretch}{1}

\title{\LARGE \bf
Assistance-Seeking in Human-Supervised Autonomy: \\ Role of Trust and Secondary Task Engagement (Extended Version)

\thanks{This work was supported in part  by the NSF award ECCS 2024649 and the ONR award N00014-22-1-2813.}
}

\author{Dong Hae Mangalindan and Vaibhav Srivastava
\thanks{D. H. Mangalindan and V. Srivastava are with the Department of Electrical and Computer Engineering, Michigan State University, East Lansing, MI 48824  e-mail: \texttt{\{mangalin, vaibhav\}@egr.msu.edu}}
}

\begin{document}

\maketitle
\thispagestyle{empty}
\pagestyle{empty}

\begin{abstract}

Using a dual-task paradigm, we explore how robot actions, performance, and the introduction of a secondary task influence human trust and engagement. In our study, a human supervisor simultaneously engages in a target-tracking task while supervising a mobile manipulator performing an object collection task. The robot can either autonomously collect the object or ask for human assistance. The human supervisor also has the choice to rely on or interrupt the robot. Using data from initial experiments, we model the dynamics of human trust and engagement using a linear dynamical system (LDS). Furthermore, we develop a human action model to define the probability of human reliance on the robot.  Our model suggests that participants are more likely to interrupt the robot when their trust and engagement are low during high-complexity collection tasks. Using Model Predictive Control (MPC), we design an optimal assistance-seeking policy. Evaluation experiments demonstrate the superior performance of the MPC policy over the baseline policy for most participants.
\end{abstract}

\section{Introduction}

Technological advancements have led to the widespread deployment of autonomous systems in industrial, commercial, healthcare, and agricultural sectors. These systems can handle hazardous or repetitive tasks, optimizing human resource allocation. However, despite their increasing efficiency and reliability, they often still need human supervision for complex tasks and in uncertain environments~ \cite{Selma2017,akash2020human,peterssupervisor2015, gupta2024structural}.

Trust in human-robot interaction is formally defined as ``the attitude that an agent will help achieve an individual’s goals in a situation characterized by uncertainty and vulnerability''~\cite{lee2004trust}. Autonomous systems are often underutilized due to the lack of trust, defeating the purpose and benefits of using automation. In contrast, an excessive reliance or trust in automation can lead to misuse or abuse of the system\cite{hoff2015trust,parasuraman1997humans}. 
Therefore, for efficient human-robot collaboration, human trust in autonomous systems should be carefully calibrated.

Supervisors often juggle multiple tasks, managing their own responsibilities while overseeing others. This dynamic also applies to human supervisors overseeing autonomous agents. Supervisors must maintain their task performance even while managing these agents. For instance, imagine a human and an autonomous robot working together in an orchard, collecting fruits. The human is tasked with both collecting fruits and ensuring the robot functions correctly. The human should only intervene with the robot when absolutely necessary to prevent compromising their own fruit collection. Similarly, the robot should be designed to operate with minimal interference to the human's tasks.

In this work, within a dual-task paradigm involving the supervision of an autonomous assistance-seeking object-collection robot, we study how the robot's assistance-seeking affects human trust in the robot and their engagement in the secondary task. We also investigate how robot's actions and performance impact supervisors' secondary task performance. We model the dynamics of human trust in the robot and secondary-task engagement using LDS, which are then utilized to design an optimal assistance-seeking policy.

To gain insight into human trust, it is important to understand the various factors that shape it. The effect of robot reliability, workload, self-confidence, and information transparency on human trust has been studied~\cite{lee1992trust, lee1994trust, manzey2012human, desai2013impact, akash2017dynamic, yang2017evaluating, Wang2015DynamicRS, chen2020trust, nikolaidis2012human, akash2020human}. 
These works highlight that negative experiences lead to a decrease in human trust. Human reliance on automation is influenced not only by trust but also by self-confidence. When individuals have lower self-confidence than their trust in automation, they are more likely to rely on automated systems~\cite {lee1994trust}. 
Additionally, a higher level of transparency, i.e., providing more information to the human, can often have a positive effect on human trust but can increase their workload~\cite{yang2017evaluating,akash2020human}.
Environmental factors such as task complexity affect human trust as well. A comprehensive investigation into factors affecting human trust in automation categorizes them into three main groups: robot-related, human-related, and environment-related factors~\cite{hancock2011meta}, and suggests that robot-related factors such as reliability have the highest effect on trust, followed by environment-related factors.

Alongside the study of factors influencing trust, several models of trust evolution have been proposed, and they are categorized into various groups. A class of these models is called probabilistic models. Probabilistic trust models treat trust as a latent state and estimate its distribution within a Bayesian framework~\cite{OPTIMO}. This distribution is conditioned on factors such as human actions, robot performance, and other contextual information relevant to the task. A subclass of such models is Partially Observable Markov Decision Process (POMDP)-based models, where the dynamics of trust are defined by the state transition functions conditioned on actions and contextual information, with observations treated as human actions~\cite{chen2020trust, nikolaidis2012human, zahedi2023trust, akash2020human, 10.5555/2906831.2906852, DBLP:journals/corr/abs-2009-11890, mangalindan2023trust}. Such models have been applied for planning to minimize human interruption~\cite{chen2020trust, mangalindan2023trust}. These models have also been used in tasks where it is needed to optimize the amount of information provided to humans, such that efficiency in the collaboration is maintained~\cite{akash2020human, 10.5555/2906831.2906852, DBLP:journals/corr/abs-2009-11890}.
 
Deterministic linear dynamical systems are used to represent trust dynamics, with trust level, cumulative trust level, and expectation bias as system states. The inputs to these models include factors that influence human trust, such as robot performance and task-specific contextual information~\cite{akash2017dynamic, yang2017evaluating, Wang2015DynamicRS}. Additionally, alongside human trust towards automation, models have been developed to capture automation's trust towards humans~\cite{rahman2018mutual, Wang2014}. Process and measurement noises have also been incorporated into linear models of human trust. To account for these noises, filtering techniques are necessary for state estimation using output measurements~\cite{azevedo2021real}. Similarly to~\cite{azevedo2021real}, we focus on a noisy linear trust model; however, our measurement model is nonlinear (binary). The latter necessitates nonlinear filtering techniques, specifically the particle filter.

Additionally, we leverage our learned model to determine and experimentally evaluate an optimal assistance-seeking policy.

Using a dual-task paradigm, we investigate how robot actions and performance impact human trust and secondary task engagement. Human participants perform both supervisory and target-tracking tasks simultaneously. We explore an assistance-seeking policy for a robot as a function of human trust, secondary task engagement, and task-specific contexts. We conduct two sets of experiments with human participants.
The initial data collection experiment involves a randomized assistance-seeking policy for the robot and random task complexity assignments. Using the collected data, we estimate models for trust dynamics and secondary task engagement using an LDS. 
In this model, the dynamics of human trust and secondary task engagement are influenced by the robot and human actions, the outcomes of those actions, and the complexities of both tasks. Additionally, we develop a model to predict human action as a function of their trust and engagement. We leverage Model Predictive Control (MPC)\cite{morari1999model} with the estimated models to compute the robot actions to maximize team performance. 
%based on the current states and task complexities. 
Subsequently, we evaluate the MPC-based assistance-seeking policy in the second set of experiments and compare it with a baseline policy.
The contributions of this work are:
\begin{itemize}
  \item the study of the effect of assistance-seeking on human trust and secondary task performance in a dual-task paradigm;
  \item the design of an optimal-assistance-seeking policy using estimated human behavior model;  
  % modeled trust and secondary task engagement dynamics, and a human behavior model trained with data collected from human participants
  \item the evaluation of the designed policy in a second set of experiments with human participants; and
  \item the use of a non-linear filter for real-time estimation of trust and secondary task engagement.
\end{itemize}

The rest of the paper is organized as follows. The dual-task setup and tasks are described in Section~\ref{sec:Exp}. We describe the human behavior model in Section~\ref{sec:Model}. We describe the design of MPC policy and its evaluation using a second set of experiments in Section~\ref{sec:policy_eval}. We discuss our findings and connect them to a broader context in Section~\ref{sec:discussion} and 
and conclude in Section~\ref{sec:concl}.

\section{Dual Task Setup: Robot Supervision and target-tracking task}
\label{sec:Exp}

In our experiment, we rely on the dual-task paradigm, wherein the human participant engages in two simultaneous tasks. The human participant supervises an autonomous mobile manipulator tasked with retrieving objects from shelves, while simultaneously performing a target-tracking task. The experiment interface is shown in Fig.~\ref{fig:exp_int}. We now define both tasks in more detail. 

\begin{figure*}[t!]
    % \vspace{0.03in}
    \centering
    \includegraphics[width=\linewidth]{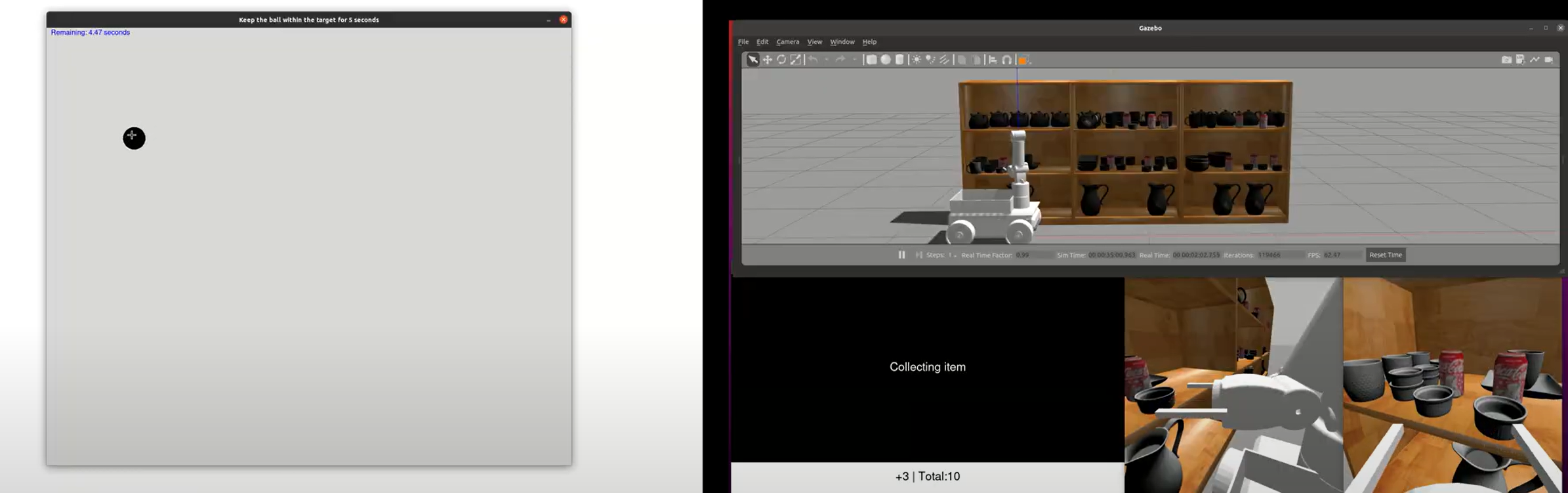}
    \caption{Experiment interface showing the target-tracking and supervisory task rendered on two adjacent screens.The setup uses ROS-Gazebo \cite{Gazebo} and resources available in \cite{downs2022google, ROSMoMa}.}
    \label{fig:exp_int}
    % \vspace{-0.15in}
\end{figure*}

\subsection{Human Supervised Robotic Object Collection Task}

In our supervisory task, the human supervisor collaborates with a robot to gather items from grocery shelves. The supervisory goal is to monitor the mobile manipulator and ensure its safety and efficiency. The human supervisor monitors a live feed displaying the robot's actions and views from two cameras placed on the end-effector (see Fig.~\ref{fig:exp_int}; right panel). 

In each trial, the mobile manipulator has the choice of attempting autonomous collection ($a^{R+}$) or requesting human assistance ($a^{R-}$), wherein the human must teleoperate the manipulator using a joystick and retrieve the item. When the robot initiates autonomous collection of an object, the human supervisor can either rely ($a^{H+}$) on the robot to collect the object autonomously, or intervene ($a^{H-}$) and collect the object via teleoperation.

 The complexity ($C^1$) of each object collection task is determined by the object's location relative to the mobile manipulator. Specifically, the complexity is high ($C^1= \supscr{C}{1,H}$) if the manipulator's direct path to the object is obstructed by an obstacle; otherwise, the complexity is low  ($C^1= \supscr{C}{1,L}$).

The outcome of each trial of the object collection task is either a success ($E^+$), if the object is collected successfully and dropped into the bin; or a failure ($E^-$), otherwise.
The human-robot team reward for the object collection task is defined by 
\begin{equation*}\label{eq:reward} \small 
\supscr{\mc R}{coll}({a^R,a^H,E})=
\begin{cases}
+3, & \text{if } (a^{R}, a^{H}, E) = (a^{R+}, a^{H+}, E^+),\\
+1, & \text{if } a^{R}=a^{H+},\\
0, & \text{if } (a^{R}, a^{H}) =  (a^{R+}, a^{H-}),\\
-4, & \text{if } (a^{R}, a^{H}, E) = (a^{R+}, a^{H+}, E^-).
\end{cases}
\end{equation*}
The objective of the human-robot team for the object collection tasks is to maximize the cumulative team reward across all trials.

\subsection{Target-tracking Task}

In addition to the supervisory task, the human participant also performs a target-tracking task (see Fig.~\ref{fig:exp_int}; left panel). In each target-tracking trial, the human participant must maintain the computer cursor inside a moving ball using the mouse for a total duration of 5 seconds.  In our experiment, we select two possible speeds (slow and normal) for the ball and they result in complexity $ C^2 \in \{\supscr{C}{$2$,slow}, \supscr{C}{$2$,norm}\}$. 

The human participants are instructed to track the ball (target) for a total cumulative time of 5 seconds. Their target-tracking performance $p$ is determined by taking the goal time of 5 seconds, dividing it by the total time it took to finish the trial, and multiplying by $100$ to get the score in \%. The participants were instructed to achieve a performance of at least $75\%$. 
\begin{equation*}\label{eq:reward} \small 
\supscr{\mc R}{track}({C^2}, p)=
\begin{cases}
+0.5, & \text{if $\supscr{C}{$2$,norm}$ and } p \geq 75\%  ,\\
+0.25, & \text{if  $\supscr{C}{$2$,slow}$ and } p \ge 75\%,\\
0, & \text{if } p <75\%.\\
\end{cases}
\end{equation*}

\subsection{Data Collection for Modeling Human Behavior}

To build a model of human supervisory behavior, we conducted a study\footnote{The human behavioral experiments were approved under Michigan State University Institutional Review Board Study ID 9452.} with $11$ human participants ($4$ males and $7$ females), performing the above supervisory and target-tracking tasks simultaneously.  Participants were recruited via e-mail for the in-person experiment. Participants were first and second-year college students not in the engineering program. An experiment's total duration was one hour and participants were compensated $\$15$ for their participation.

For the supervisory object collection tasks, each participant performed $30$ low-complexity and $30$ high-complexity trials, randomly distributed throughout the experiment.
When the manipulator operates autonomously, it has a success probability of $\supscr{p}{suc}_H=0.75$ in $\supscr{C}{1,H}$ and $\supscr{p}{suc}_L=0.96$ in $\supscr{C}{1,L}$. In each trial, the robot asks for human assistance with a probability of 0.3 in $\supscr{C}{1,H}$ trials, and with a probability of 0.1 in $\supscr{C}{1,L}$ trials. After each trial, participants reported their trust in the autonomous mobile manipulator through an interface with an 11-point scale. 
For each object collection trial, the data collected are the complexity $C^1$, robot action $a^R$, human action $a^H$, outcome $E$, and the reported human trust $y^T$.

For the target-tracking task, participants performed $30$ slow-speed trials and $30$ normal-speed trials. To ensure that all participants do the same number of trials per complexity, an initial set of $60$ such trials was selected and trials were randomly permuted for each participant and independently of the object-collection task.  
In each trial, the data collected are the complexity of the trial $C^2$, and the performance $p$.

\subsection{Data Summary}

For the supervisory task, after a successful autonomous collection, the mean trust ratings of participants are $7.38$ for low complexity and $7.36$ for high complexity, with an overall mean trust rating of $7.37$ across both complexities. After a failure, the mean trust ratings were $6.36$ for low complexity and $6.68$ for high complexity, with an overall mean of $6.6$ across both complexities. When the robot asked for assistance, the mean trust ratings were $7.3$ for low complexity and $6.98$ for high complexity, averaging $7.03$ for both complexities combined. However, when the human voluntarily interrupted, the mean trust rating dropped to $6.29$ when they interrupted the robot in high-complexity, $5$ in low-complexity, and $6.17$ across both complexities.

For the tracking task, at slow speeds, the mean performance was $p=89\%$ with $90\%$ of the trials having $p\ge75\%$. At normal speeds, the mean performance dropped to $p=82\%$ with $78\%$ of the trials having $p\ge75\%$. As expected, $p<75\%$ whenever participants interrupted the robot.

\section {Human Behavior Model}
\label{sec:Model}

In this section, we focus on a model of human supervisory behavior, i.e., a model that predicts when the human supervisor may intervene in the autonomous operation of the manipulator. We assume that human action is modulated by their trust in the mobile manipulator and their engagement in the target tracking task. 

In the following, we denote trust at the beginning of trial $t$ by $T_t \in \real$, the target-tracking engagement in trial $(t-1)$ by $G_t \in [0,10]$, the complexities of object collection and target-tracking task by $C^1_t$ and $C^2_t$, respectively, the human and robot actions during collaborative object collection in trial $t$ by $a^H_t$ and $a^R_t$, respectively, and the outcome of the object collection in trial $t$ by $E_{t+1}$.

\subsection{Human Trust Dynamics}
\label{subsec:Human_trust}

The evolution of human trust in the robot is known to be influenced by several factors such as the current trust level,  robot performance, robot actions, environmental complexity, and human actions. 
Accordingly, we assume that trust $T_{t+1}$ at the beginning of trial $t+1$ is influenced by the previous trust $T_t$, robot action $a^R_t$, collection task complexity $C^1_t$, human action $a^H_t$, and the resulting outcome of the object collection task $E_{t+1}$.

Consider the following events corresponding to 6 different combinations of $C^1$, $a^R$, and $E$; and the additional case when human interrupts the robot corresponding to $(a_t^R, a_t^H)=(a^{R+},a^{H-})$. Note that for brevity, we have defined the events with their outcomes. 
\medskip

\resizebox{0.95\linewidth}{!}{
\centering
\begin{tabular}{ |c|l||c|l|} 
\hline 
$\theta^1_t$ & $(C^{1,L}_t,a^{R+}_t,a^{H+}_t,E^+_{t+1})$ & $\theta^5_t$ & $(C^{1,H}_t,a^{R+}_t,a^{H+}_t,E^-_{t+1})$ \\ 
$\theta^2_t$ & $(C^{1,L}_t,a^{R+}_t,a^{H+}_t,E^-_{t+1})$ & $\theta^6_t$ & $(C^{1,H}_t,a^{R-}_t)$   \\ 
$\theta^3_t$ & $(C^{1,L}_t,a^{R-}_t)$  & $\theta^7_t$ & $(a^{R+}_t,a^{H-}_t)$ \\ 
$\theta^4_t$ & $(C^{1,H}_t,a^{R+}_t,a^{H+}_t,E^+_{t+1})$ &&\\ 
\hline
\end{tabular}}

\medskip

Let $u^T_t \in \{0,1\}^7$ be an indicator vector representing one of the above 7 possible events in collaborative object collection trial $t$.
We model the trust dynamics using the following LDS 
\begin{equation}\label{eq:trust-dynamics}
    T_{t+1}=A^TT_t+B^Tu^T_t+v^T_t,
\end{equation}
where $A^T \in (0,1)$, $B^T \in [-1,1]^{1\times 7}$, and 
$ v^T_t \sim \mathcal{N}(0,\,\sigma^2_T)\,, t \in \natural$ are i.i.d. realizations of zero-mean Gaussian noise. The entries of $B^T$ determine the influence of the events $\theta^i_t$ on the evolution of trust. 

We assume that the trust reported by the participants after each trial is possibly a scaled and noisy measurement of the trust $T_t$, i.e., 
\begin{equation}\label{eq:trust-measurement}
    y^T_t=C^TT_t+w^T_t,
\end{equation}
where $C^T\in [0,1]$ and $ w_t^T \sim \mathcal{N}(0,\,\sigma_y^2)\,, t \in \natural$ are i.i.d. realizations of zero-mean Gaussian noise.

The trust model parameters include $A^T, B^T, C^T, \sigma^2_T$ and $\sigma^2_y$. To estimate these parameters from the reported trust values $y^T_t$, we adopt the Expectation-Maximization (EM) algorithm~\cite{pml2Book}.  
The estimated parameters of the trust dynamics model are 
\begin{align*}
   A^T &= 0.92, \quad \sigma^2_T  = 0.22,\quad C^T = 1.00,\quad \sigma^2_y  =  0.22,\\
   B^T & = \setlength{\arraycolsep}{4.5pt}\begin{bmatrix}
  0.76 &  -0.38 &  0.26 &0.78 & -0.43 & 0.52 & -0.12
\end{bmatrix}.
\end{align*}

The elements of $B^T$ represent the effect of each scenario on human trust. The estimated values for contexts associated with a negative experience, e.g., a failed collection or human intervention are negative. Similarly, the estimated values for contexts associated with a positive experience, e.g., successful collections or when the robot requests assistance, are positive. 
Interestingly, as found in our previous study~\cite{mangalindan2023trust}, asking for assistance can help increase human trust, which is also consistent with the estimated $B^T$. In other words, the estimates suggest that   
successful collection increases and maintains trust; failed collections decrease it; and asking for assistance can help repair and increase trust. Additionally, human interruptions decrease their trust.

\subsection{Human Target-tracking Engagement Dynamics}
\label{subsec:Human_per}

Similar to the human trust dynamics, we seek to model the influence of robot action $a^R_t$, the complexity of tracking task $C^2_t$, and previous experience in the supervisory object collection task $E_{t}$ on the evolution of human target-tracking engagement and performance. The experience $E_t$  is the outcome of the trial for autonomous collection; while, for engagement dynamics, we classify it as $E^-$ if the human interrupts the robot, and as $E^+$, if the robot asks for assistance. 

Consider the following events corresponding to 8 different combinations of $C^2$, $a^R$, and  $E$. Note that for brevity, we have defined the events with their outcomes.
\medskip 

\resizebox{0.95\linewidth}{!}{
\centering
\begin{tabular}{ |c|c||c|c|} 
\hline 
$\phi^1_t$ & $(\supscr{C}{2,slow}_t,a^{R-}_t,E^+_t)$ & $\phi^5_t$ & $(\supscr{C}{2,slow}_t,a^{R+}_t,E^+_t)$ \\ 
$\phi^2_t$ & $(\supscr{C}{2,norm}_t,a^{R-}_t,E^+_t)$ & $\phi^6_t$ & $(\supscr{C}{2,norm}_t,a^{R+}_t,E^+_t)$   \\ 
$\phi^3_t$ & $(\supscr{C}{2,slow}_t,a^{R-}_t,E^-_t)$  & $\phi^7_t$ & $(\supscr{C}{2,slow}_t,a^{R+}_t,E^-_t)$ \\ 
$\phi^4_t$ & $\supscr{C}{2,norm}_t,a^{R-}_t,,E^-_t)$ & $\phi^8_t$ & $(\supscr{C}{2,norm}_t,a^{R+}_t,E^-_t)$\\ 
\hline
\end{tabular}}

\medskip 

Let $u^G_t \in \{0,1\}^8$ be an indicator vector representing one of the above 8 possible events in the target-tracking task during trial $t$.
We model the human target-tracking engagement dynamics as an LDS 
\begin{align}\label{eq:performance-dynamics}
    G_{t+1}&=A^GG_t+B^Gu^G_t+v^G_t, \\\label{eq:performance-measurement}
    p_{t}&=C^GG_t+w^G_t,  
\end{align}
where $A^G \in [0,1]$, $B^G \in [0,10]^{1\times 8}$,
$ v^G_t \sim \mathcal{N}(0,\,\sigma^2_G)\,, t \in \natural$ are i.i.d. realizations of zero-mean Gaussian noise, 
$C^G\in [0,10]$, and  $ w^G_t \sim \mathcal{N}(0,\,\sigma^2_p)\,, t \in \natural$ are i.i.d. realizations of zero-mean Gaussian noise. 
The entries of $B^G$ determine the influence of the events $\phi^i_t$ on the human engagement in the tracking task. 

The engagement model parameters include $A^G, B^G, C^G, \sigma^2_G$ and $\sigma^2_p$. Considering $p$ as the tracking task result of the human, we adopt the EM algorithm to estimate these parameters.  

\begin{align*}
   A^G &= 0.19, \quad \sigma^2_G  = 1.44, \quad C^G = 9.96, \quad
 \sigma^2_p  =  3.79, \\
   B^G & = \setlength{\arraycolsep}{4.5pt}\begin{bmatrix}
  7.47 & 6.72 & 7.24 & 6.38 & 7.30 & 6.51 & 7.06 & 6.59
\end{bmatrix}.
\end{align*}
The elements of $B^G$ represent the effect of each scenario on human tracking-task engagement and performance.

The estimated $B^G$ indicates that whenever the robot asks for assistance, {for the majority of cases}, the associated estimated values are greater than the values associated with the cases when the robot collects autonomously. This is expected as whenever the robot asks for assistance, the human supervisor can focus more on their tracking task and finish it before teleoperating the robot, knowing that the robot has stopped and will not move unless they take control of it. In contrast, when the robot attempts to collect autonomously, the human supervisor may switch attention from the tracking task to the supervisory task to decide whether to rely on the robot. The estimated parameters are higher when the target speed is slow compared to when it is normal. Whenever the previous experience with the robot is $E^-$, the parameters are lower compared to the case of $E^+$. This is consistent with the fact that whenever humans perceive a failure, they tend to pay more attention to the supervisory task in the future.

\subsection{Human Action Model}
\label{subsec: Human_behavior}

We now focus on modeling the probability of the human relying on the robot when the robot attempts to collect autonomously. We assume that this probability depends on the object collection task complexity $C^1_t$, human trust $T_t$, and human target-tracking engagement $G_t$, which is representative of the attention they put on the supervisory task. We model the human action probability as a sigmoid function 
\begin{align}\label{eq:action-model}
 \prob(a^{H+}|T,G,C^1)=\frac{1}{1+e^{-(a_T^{C^1}T+a_G^{C^1}G+b^{C^1})}}.
\end{align}
The model parameters are learned with Maximum Likelihood Estimation using the Monte Carlo simulation method, resulting in the following parameter values for different object collection task complexities
\begin{center}
\begin{tabular}{ |c|c|c|c|} 
\hline
& $a_T^{C^1}$ & $a_G^{C^1}$ & $b^{C^1}$ \\
\hline
$C^{1,L}$ & 0.09 & 0.08  & 3.6 \\ 
\hline
$C^{1,H}$ & 0.20 & 0.40 & -2.7 \\ 
\hline
\end{tabular}
\end{center}

The probabilities of reliance for each object collection task complexity are shown in Fig.~\ref{fig:Human_act}. For low complexity $C^{1,L}$, the probability of reliance is high and close to 1, regardless of $T$ and $G$. For high complexity $C^{1,H}$, there is a lower probability of reliance in general as compared to low complexity $C^{1,L}$. It can be seen that the probability of reliance increases with trust as well as with tracking task engagement. The latter may be caused by the human focusing more on the tracking task as compared to the supervisory task, thereby increasing their probability of reliance on the mobile manipulator.

\begin{figure}[h!]
    \centering
    \includegraphics[width=4cm, angle=-90]{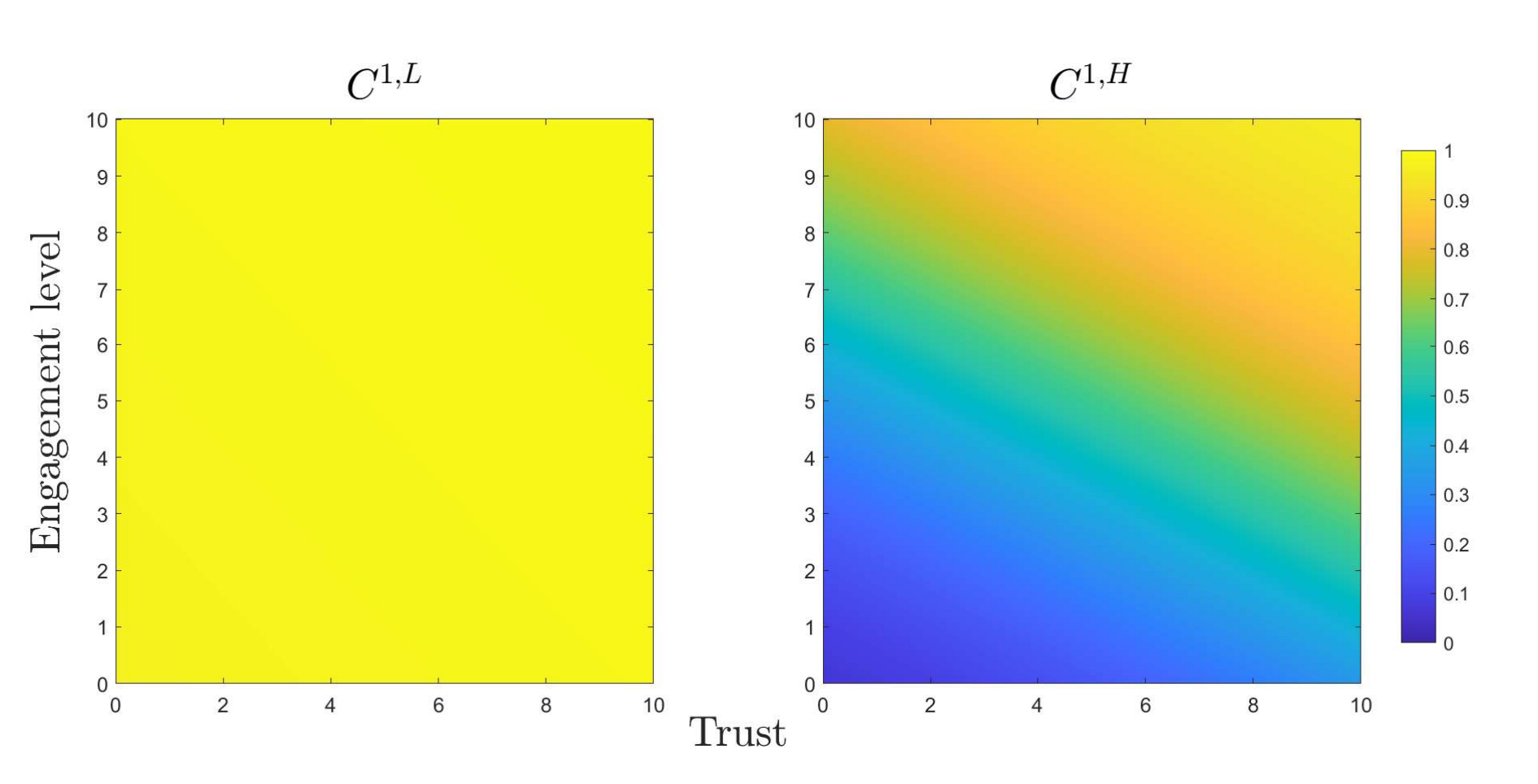}
    \caption{Human action model: Probability of reliance $\prob(a^{H+}|T,G,C^1,a^{R+})$ conditioned on object collection task complexity $C^1$,$T$, and $F$}
    \label{fig:Human_act}
\end{figure}

Fig.~\ref{fig:diagram} illustrates the relationship between the exogenous factors, trust dynamics, tracking task engagement dynamics, and human action model.

\begin{figure}[h!]
    % \vspace{0.03in}
    \centering
    \includegraphics[width=0.8\linewidth]{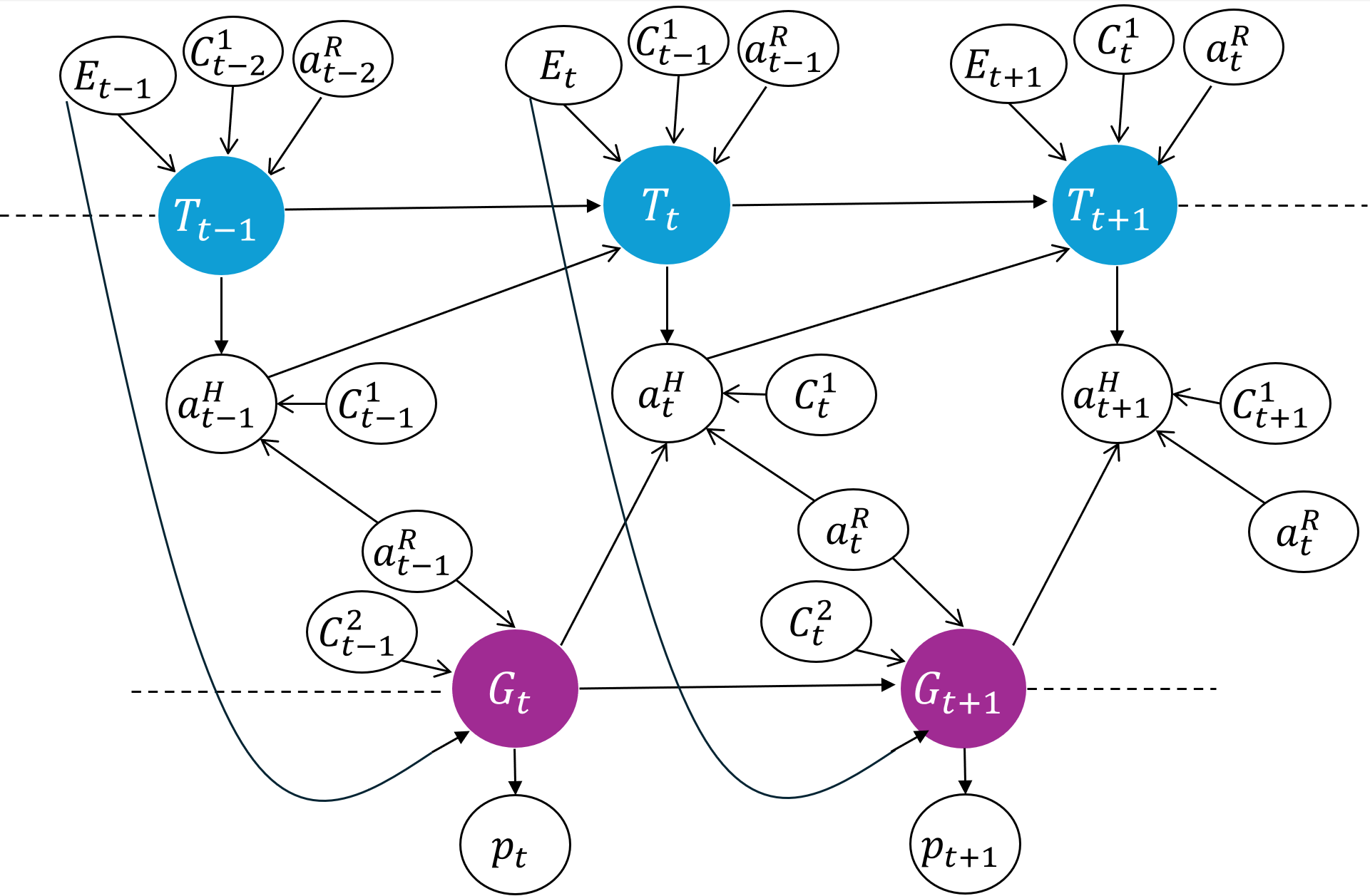}
    \caption{Human trust and target tracking engagement dynamics. Human trust and tracking task engagement modulate human action.}
    \label{fig:diagram}
    % \vspace{-0.15in}
\end{figure}

%\clearpage 

\section{Design of human assistance seeking policy}
\label{sec:policy_eval}
In this section, we propose an optimal assistance-seeking policy, leveraging the estimated trust dynamics, target-tracking dynamics, and human action model.

\subsection{Expected Reward and  Certainty-equivalent Dynamics}
Since the choice for the robot is to perform autonomous collection or to seek assistance in each trial, we first write the reward for each trial as a function of these actions and the expected evolution of the trust and engagement dynamics with only the robot action as an exogenous input. In the following, we assume that the complexities of the object collection and target tracking tasks are i.i.d. Bernoulli random variables. We assume for each trial, the probability of object collection complexity being high is $\beta_1\in [0,1]$ and the probability of target tracking speed being normal is $\beta_2 \in [0,1]$.

\begin{lemma}[\bit{Expected rewards}] \label{lem:expt-rewards}
     Given the probability of successful autonomous collection by the robot, $\supscr{p}{suc}_{C^1}$, the probabilities of human action  $\prob(a^{H}|T,G ,C^1), \ a^{H} \in \{a^{H+}, a^{H-}\}$, for $C^1 \in \{\supscr{C}{1,L}, \supscr{C}{1,H}\}$, and the probability of an object-collection trial being high complexity, $\beta_1$,    
    the expected rewards for action $a^{R+}$ in the supervisory task is 
    \begin{align}
    \supscr{{\mc J}}{coll}({T_t,G_t,a^{R+}_t})&=(1-\beta_1)\supscr{{\mc R}}{coll}_{C^{1,L}}(T_t,G_t,a^{R+}_t) \notag\\
    &+ \beta_1 \supscr{{\mc R}}{coll}_{C^{1,H}}(T_t,G_t,a^{R+}_t), \label{eq:expt-reward-collection}
    \end{align}
 where 
    \begin{align*}
    \supscr{{\mc R}}{coll}_{C^1}(T_t,G_t,a^{R+}_t) &=\prob(a^{H+}|T_t,G_t,C^1_t)\supscr{{\mc R}}{auto-coll}_{C^1}(a^{R+}, a^{H+})\\
    &+\prob(a^{H-}|T_t,G_t,C^1_t)\supscr{{\mc R}}{coll}(a^{R+},a^{H-}, *),
    \end{align*}
is the expected reward for taking action $a^{R+}$ in $C^1$, considering possible human actions $a^H$, and
    \begin{align*}
    \supscr{{\mc R}}{auto-coll}_{C^1}(a^{R+}, a^{H+}) &=\supscr{p}{suc}_{C^1}\supscr{{\mc R}}{coll}(a^{R+},a^{H+},E^+)\\
    &+(1-\supscr{p}{suc}_{C^1})\supscr{{\mc R}}{coll}(a^{R+},a^{H+},E^-),
    \end{align*}
is the expected reward for an autonomous collection in complexity $C^1$.

Similarly, for the probability of the target-tracking task being normal speed, $\beta_2$,  the expected reward $\supscr{{\mc J}}{track}({G_t,a^R_t})$ in the target-tracking task is 
    \begin{align*}
        \supscr{{\mc J}}{track}({G_t,a^R_t})&=(1-\beta_2)\supscr{{\mc R}}{track}(\supscr{C}{2,slow},p_{t+1})\\
        &+ \beta_2 \supscr{{\mc R}}{track}(\supscr{C}{2,norm},p_{t+1}). 
    \end{align*}

\end{lemma}
\begin{proof}
    The statement follows from straightforward expected value calculations. 
\end{proof}

\medskip 

We now define the certainty-equivalent trust and engagement dynamics in which only the robot action is the input. For trial $t$, we define $u_t= \begin{bmatrix}q_t& 1- q_t \end{bmatrix}^\top$, where  $q_t =\prob(a^R_t = a^{R+})$ is the probability with which the robot attempts to collect autonomously.

Let the probability of event $\theta^i_t$ conditioned on the value of $a^R_t$ stated in the event be $\prob(\theta^i_t), i \in\until{7}$ defined by
\begin{align*}
\prob(\theta^1_t) &= (1-\beta_1) \supscr{p}{suc}_{\supscr{C}{1,L}} \prob(a^{H+}|T,G ,C^1), \\   
\prob(\theta^2_t) &= (1-\beta_1) (1-\supscr{p}{suc}_{\supscr{C}{1,L}} )\prob(a^{H+}|T,G ,C^1),\\
\prob(\theta^3_t) &= (1-\beta_1), \\
\prob(\theta^4_t) &= \beta_1 \supscr{p}{suc}_{\supscr{C}{1,L}} \prob(a^{H+}|T,G ,C^1), \\
\prob(\theta^5_t) &= \beta_1 (1-\supscr{p}{suc}_{\supscr{C}{1,L}} )\prob(a^{H}|T,G ,C^1), \\
\prob(\theta^6_t) &= \beta_1, \\
\prob(\theta^7_t) &= \prob(a^{H-}|T,G ,C^1). 
\end{align*}

Likewise, let the probability of event $\phi^i_t$ conditioned on the value of $a^R_t$ stated in the event be $\prob(\phi^i_t), i \in\until{8}$ defined by

\begin{center}  
\begin{tabular}{ l|l } 
 $\prob(\phi^1_t) = (1-\beta_2)\epsilon_t ,$ & $\prob(\phi^5_t) = (1-\beta_2)\epsilon_t,$ \\ 
 $\prob(\phi^2_t) = \beta_2\epsilon_t,$ & $\prob(\phi^6_t) = \beta_2\epsilon_t,$ \\ 
 $\prob(\phi^3_t) = (1-\beta_2)(1-\epsilon_t),$ & $\prob(\phi^7_t) = (1-\beta_2)(1-\epsilon_t),$ \\
 $\prob(\phi^4_t) = \beta_2(1-\epsilon_t),$ & $\prob(\phi^8_t) = \beta_2(1-\epsilon_t),$\\
\end{tabular}
\end{center}  
where $\epsilon_t$ is the probability of previous experience being $E^+$ and is defined by
\begin{multline*}
\epsilon_t = 
q_{t-1} \expt_{C^1_{t-1}\sim \text{Ber}(\beta_1)}\big[\supscr{p}{suc}_{C^1}\prob(a^{H+}|T_{t-1},G_{t-1},C^1_{t-1}))\big] \\+(1-q_{t-1}), 
\end{multline*}
where we assume the previous experience as $E^-$ if the human interrupts the robot, and as $E^+$, if the robot asks for assistance.

\begin{lemma}[\bit{Certainty-equivalent evolution}]\label{lem-ce-evolution}
Given the probability of successful autonomous collection by the robot, $\supscr{p}{suc}_{C^1}$, the probabilities of human action  $\prob(a^{H}|T,G ,C^1), \ a^{H} \in \{a^{H+}, a^{H-}\}$, for $C^1 \in \{\supscr{C}{1,L}, \supscr{C}{1,H}\}$, the certainty-equivalent evolution of the trust $\bar T_t$ and engagement $\bar G_t$ dynamics and associated probability of intervention $\bar \prob_t$ and tracking performance $\bar p_{t}$ are
\begin{align*}
    \bar T_{t+1} &={A}^T \bar T_t+\hat{B}^T_t u_t\\
    \bar G_{t+1}&={A}^G \bar G_t+\hat{B}^G_t u_t \\
    \bar \prob_t & = \expt_{C^1 \sim\text{Ber}(\beta_1)} \big[\prob(a^{H+}|\bar T_t,\bar G_t,C^1)\big] \\
    \bar p_{t}& =C^G \bar G_t,
\end{align*}
where
\begin{align*}
    \hat{B}^T_t & = \begin{bmatrix}
    \sum_{i\notin\{3,6\}} \prob(\theta^i_t) (B^T)_{1,i} & \sum_{i\in\{3,6\}} \prob(\theta^i_t) (B^T)_{1,i}
\end{bmatrix}\\
    \hat{B}^G_t &=
    \begin{bmatrix}
    \sum_{i=5}^8 \prob(\phi^i_t)(B^G)_{1,i} & \sum_{i=1}^4 \prob(\phi^i_t)(B^G)_{1,i}
    \end{bmatrix}.
\end{align*}
\end{lemma}
\smallskip 
\begin{proof}
The entries in $\hat{B}^T_t$ and $\hat{B}^G_t$ correspond to the expected influence of the autonomous collection attempt and assistance-seeking on trust and engagement, respectively. They are obtained by averaging entries of $B^T$ and $B^G$, respectively, associated with these events weighted by their probabilities. 
\end{proof}

\subsection{MPC-based Assistance-seeking Policy}

Using the models described in Section~\ref{sec:Model}, we design an optimal assistance-seeking policy using Model Predictive Control (MPC)~\cite{rawlings2017model}. We adopt a certainty-equivalent MPC approach~\cite{mattingley2011receding} that considers the expected reward and deterministic system dynamics described in Lemmas~\ref{lem:expt-rewards}~and~\ref{lem-ce-evolution}, respectively. 
Specifically, the following optimization problem is solved at the beginning of each trial $t$ with trust and target-tracking engagement $T_t$ and $G_t$, respectively.  

\begin{equation}\label{eq:mpc} 
\begin{split}
    \underset{{q_t}, \ldots {q_{t+N-1}}}{\text{maximize}} \quad & \sum_{\tau=t}^{t+N-1} \supscr{\bar{\mc R}}{coll}({\bar T_\tau, \bar G_\tau,q_\tau})+\supscr{\bar{\mc R}}{track}({\bar G\tau,q_\tau}) \\
    \text{s.t.}\quad  & \bar T_{\tau+1}={A}^T \bar T_\tau+ \hat{B}^T_\tau u_\tau, \ \bar T_t = T_t\\
    & \bar G_{\tau+1}={A}^G \bar G_\tau+\hat{B}^G_\tau u_\tau\ ,\bar G_t = G_t \\
    & \bar \prob_T  = \expt_{C^1 \sim\text{Ber}(\beta_1)} \big[\prob(a^{H+}|\bar T_t,\bar G_t,C^1)\big] \\
    & \bar p_{t} =C^G \bar G_t \\
    & q_\tau \in [0,1],
\end{split}
\end{equation}

for $\tau=t,\ldots,(t+N-1)$, 
where $u_\tau= \begin{bmatrix}q_\tau& 1- q_\tau \end{bmatrix}^\top$, and the expected rewards are calculated as
\begin{align*}
    \supscr{\bar{\mc R}}{coll}({T_\tau,G_\tau,a^R_\tau})&=q_\tau \supscr{{\mc J}}{coll}_{C^1}(T_\tau,G_\tau,a^{R+})\\
    &+(1-q_\tau)\supscr{\mc R}{coll}({a^{R-},*,*}),
\end{align*}
\begin{align*}
    \supscr{\bar{\mc R}}{track}({G_\tau,q_\tau}) &=q_\tau\supscr{{\mc J}}{track}({G_\tau,a^{R+}_\tau})\\
    &+(1-q_\tau)\supscr{{\mc J}}{track}({G_\tau,a^{R-}_\tau}).
\end{align*}

Since, at the beginning of trial $t$, the complexities $C^1_t$ and $C^2_t$ are known, for $\tau=t$, we set $\beta_1$ and $\beta_2$ to 0 or 1, depending on the context, in Lemmas~\ref{lem:expt-rewards}~and~\ref{lem-ce-evolution} for use in optimization problem~\eqref{eq:mpc}.
Likewise, for $\hat{B}^P_t$, $E_t$ is known and we set $\epsilon_t$ to 0 or 1, depending on the outcome. 
% $q_{t-1}$ is known required to calculate the entries of $\hat{B}^P_\tau$.

In the MPC framework, the optimization problem~\eqref{eq:mpc} is solved at the beginning of each trial $t$, the optimal sequence of robot actions $\{\bar q_t, \ldots \bar q_{t+N-1}\}$ is computed and only the first action is executed, i.e., $a^R_t$ is chosen as the autonomous collection with probability $\bar q_t$, and seeking assistance, otherwise. 
The look-ahead horizon $N$ is a tuning parameter that is selected to balance the performance and computational time and to account for the accuracy of the look-ahead model.  We solved the optimization problem~\eqref{eq:mpc}, using MATLAB's \texttt{fmincon} function. The assistance-seeking policy obtained using the above MPC formulation with a look-ahead horizon $N=5$ is shown in Fig.~\ref{fig:Robot_pol}.

In low-complexity object collection tasks, the optimal action for the robot is always to collect autonomously, regardless of the tracking task complexity, previous experience, previous tracking performance, or trust level. This outcome is anticipated because of the low interruption rate and high success rate observed in low-complexity object collection tasks.

% In high-complexity object collection tasks, the optimal action is to ask for assistance if secondary task engagement $P_t$ falls below a threshold of approximately $P_t=5$, regardless of tracking task complexity, previous experience, or trust level. 
In high-complexity object collection tasks, the computed optimal action is to seek assistance if secondary task engagement $G_t$ falls below $G_t=5$, regardless of other variables.
Since low target-tracking performance may be caused by the human paying more attention to the supervisory task, the optimal policy seeks assistance at low target-tracking engagement; thereby allowing the human to finish the tracking task and then assist the manipulator. 
In contrast, when tracking task engagement is high, the optimal policy depends on human trust: it asks for assistance at low trust levels and collects autonomously for high trust levels. The threshold of trust for seeking assistance decreases with target-tracking performance. Since higher target-tracking performance is an outcome of higher engagement with the tracking task resulting in fewer interruptions by the supervisor, the optimal policy attempts autonomous collection of the object.

Interestingly, for high-complexity object collection and normal speed tracking tasks,  
the threshold of trust for asking assistance in the optimal policy is lower when the previous experience is negative ($E^-$) compared to when it is positive ($E^+$). Since the probability of interruption increases following a failure, the policy aims to mitigate this interruption by requesting assistance, which results in human trust repair and enables human to improve their tracking performance.

\begin{figure}
\begin{subfigure}{.5\textwidth}
  \centering
  \includegraphics[width=4cm]{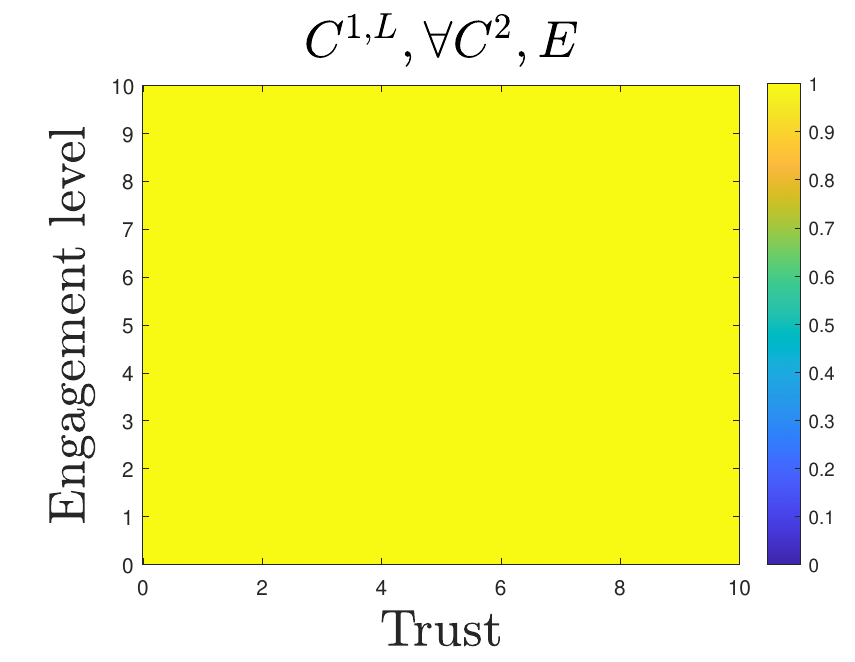}
  \caption{Optimal action for $C^1=\supscr{C}{1,L}$}
  \label{fig:sfig1}
\end{subfigure}
\begin{subfigure}{.5\textwidth}
  \centering
  \hspace*{0.25cm}
  \includegraphics[width=7cm]{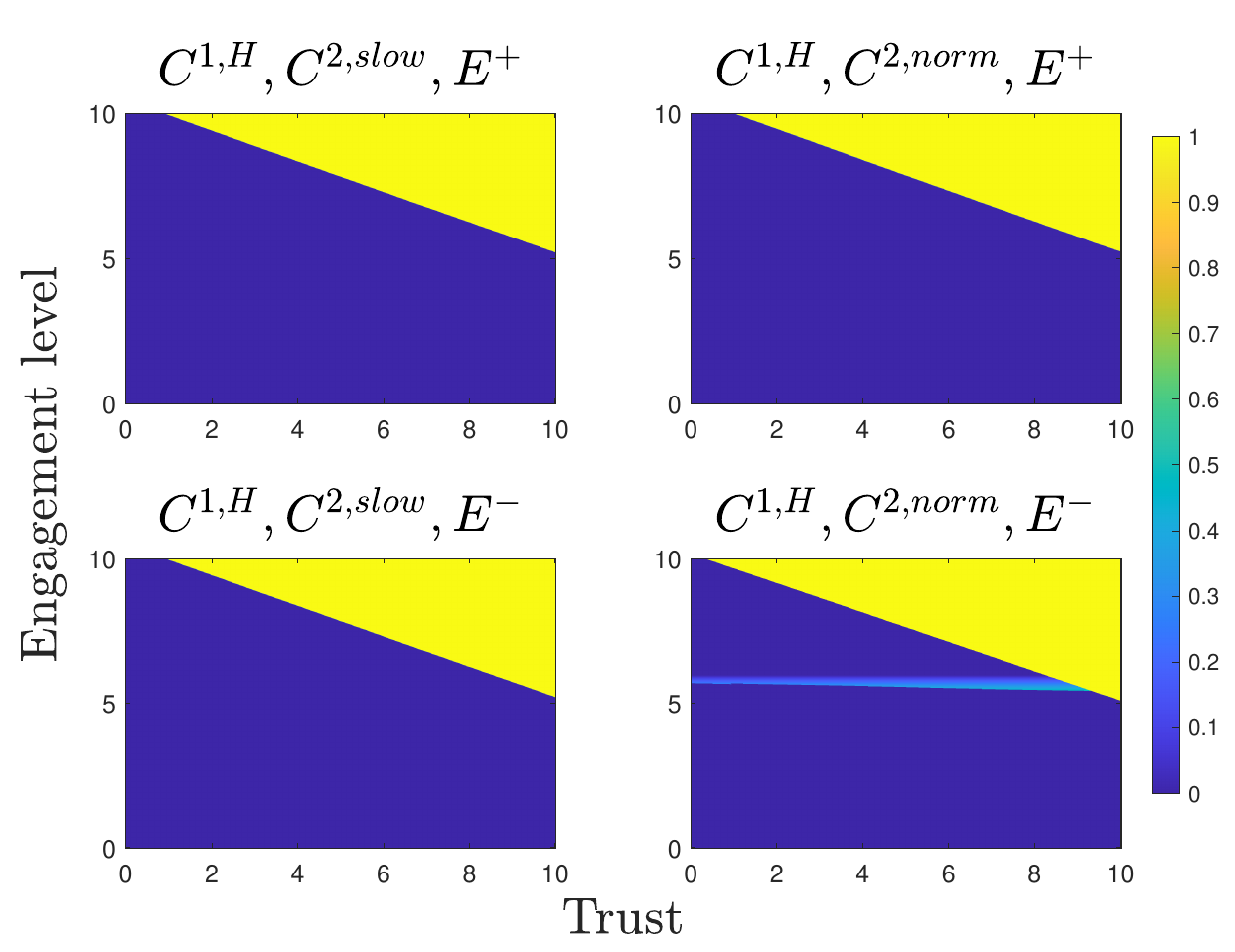}
  \caption{Optimal action for $C^1=\supscr{C}{1,H}$}
  \label{fig:sfig2}
\end{subfigure}
\caption{MPC policy: The optimal action for $\supscr{C}{1,L}$ is to collect, whereas for $\supscr{C}{1,H}$, action switches based on the trust and engagement states.}
\label{fig:Robot_pol}
\end{figure}

\subsection{Evaluation Experiment}

We conducted a second set of human experiments to evaluate the designed assistance-seeking policy. For comparison, we also computed a baseline policy that only accounts for the object-collection complexity and the tracking speed. 
Using data from the initial data collection experiment, we estimated the probability of reliance $\prob(a^{H+}|C^1, C^2)$ on the mobile manipulator, and the probability of success $\prob(p\geq0.75|C^1, C^2)$ in the tracking task. We computed the expected reward for action $a^{R+}$ and $a^{R-}$ for each possible value of $(C^1, C^2)$. We deduce that the optimal baseline policy for the mobile manipulator is to always attempt autonomous collection. 
We refer to this baseline policy as ``greedy policy".

We recruited $5$ participants, and each one completed two blocks of experiments, during which the robot followed the MPC and greedy policies. The order of the two blocks was randomized. In each block, participants engage in 
15 trials each of low-complexity and high-complexity supervisory object collection tasks. Half of these trials involved simultaneous execution of slow-speed target tracking tasks,  while the remaining involved normal-speed target tracking. Sequences of object collection and tracking tasks with the above-mentioned complexities were selected and these sequences were randomly permuted for every participant. 

To match a realistic scenario in which the participants may not report their trust after each trial, we did not use the reported trust values and estimated trust and tracking task engagement of the human participant using their actions to rely or intervene, and the measured tracking performance $p_t$. 
To this end, we adopted a particle filter with human trust dynamics~\eqref{eq:trust-dynamics} and engagement dynamics~\eqref{eq:performance-dynamics}, and using the human action model~\eqref{eq:action-model} and target-tracking task performance output~\eqref{eq:performance-measurement} to estimate $T_{t+1}$ and $G_t$.

The cumulative reward for all participants under both the greedy and the MPC policy is illustrated in Fig.~\ref{fig:score}. The score of a trial is calculated as the sum of rewards in tracking and collection tasks. The median scores are  $65.75$ and $57$ for the MPC and the greedy policy, respectively. The MPC policy outperforms the greedy policy for most participants. The cumulative number of interruptions in our MPC policy, $16$, is also lesser than the greedy policy, $23$.

\begin{figure}[h!]
    % \vspace{0.03in}
    \centering
    \includegraphics[width=7cm]{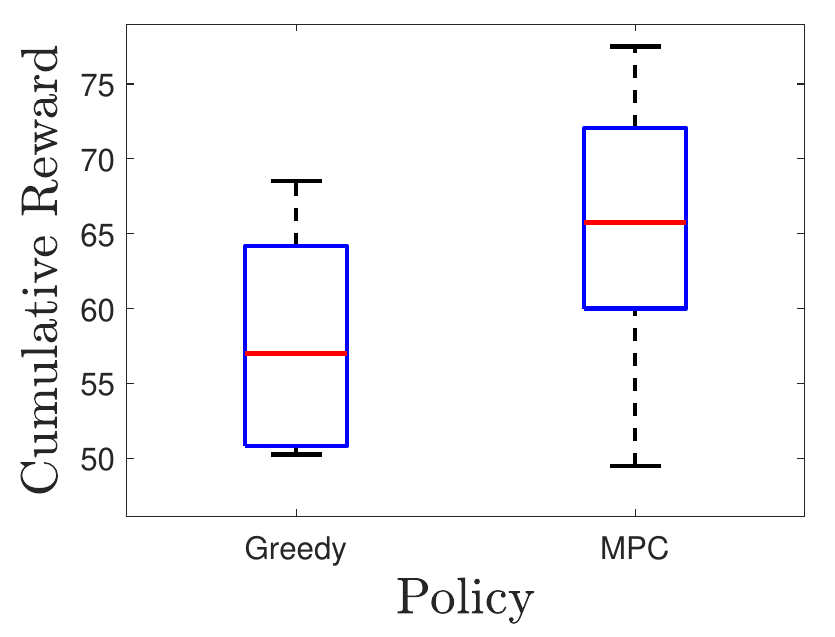}

    \caption{Cumulative reward (sum of collection and tracking rewards) statistics for both policies}
    \label{fig:score}
    % \vspace{-0.15in}
\end{figure}

\section{Discussion}\label{sec:discussion}

When humans encounter a multitasking scenario, they must optimize their resources to achieve acceptable performance in every task. In the context of a human-robot team, human trust in the robot and their engagement in their own responsibilities play a key role in their multi-tasking performance. 
To understand and leverage human trust and engagement dynamics toward designing better policies for the robot, we developed a model and estimated it using human participant experiment data. 

In our prior study~\cite{mangalindan2023trust}, we 
investigated the supervisory object collection task without any secondary task. In contrast to the dual-task paradigm, we observed a higher interruption rate in our prior study (0.173 vs. 0.095). This is possibly due to the division of engagement across two tasks. 
Additionally, the optimal assistance-seeking policy in the single-task paradigm, as described by~\cite{mangalindan2023trust}, seeks assistance only in cases of high task complexity and low human trust. In this paper, the robot's actions are guided by both the level of human trust and engagement in a secondary task. When engagement with the secondary task is low, resembling a single-task paradigm, the optimal policy consistently seeks assistance. The optimal policy considers the secondary task performance, which tends to decline at low engagement levels, and seeking assistance allowing humans to satisfactorily complete the secondary task before assisting.
 
Our model parameters indicate that failures decrease trust, aligning with the literature that negative experiences negatively affect trust~\cite{desai2013impact,yang2017evaluating}. Similarly, when humans interrupt and intervene, their trust in the robot tends to decrease~\cite{chen2020trust}.
Our human action model suggests that humans are unlikely to rely on robots in scenarios associated with high complexity and low trust. This is consistent with the observations in the literature that at a higher risk, a higher trust is required for humans to rely on the robot~\cite{chen2020trust,akash2020human}.

In a dual-task setup, when trust in the robot is low, they tend to interrupt the robot, reducing their attention to the other task and thereby diminishing their performance. When a human supervises an autonomous robot, it has been established that as human trust towards the robot increases, their rate of monitoring and interruption decreases~\cite{zahedi2023trust}. 
When performing a dual task, if performance on one task diminishes, it is due to the other task consuming most of the operator's resources~\cite{casner2010measuring}.
Therefore, we can consider their engagement in the secondary task as a measure of their trust. When their trust in the robot is low, they tend to supervise it more closely, diverting attention from their tracking task, thus decreasing their engagement and diminishing their performance.

\section{Conclusions}
\label{sec:concl}
Using a dual-task paradigm, we focused on a scenario in which a human supervises object collection by a robot while performing a secondary target-tracking task. We studied the impact of robot performance on the primary task and human performance on the secondary task on the evolution of trust and secondary task engagement. We modeled the trust and secondary task performance dynamics using linear dynamical systems with Gaussian noise and human action selection probability as a static function of their trust and secondary task engagement. A data-collection experiment with human participants was conducted and the collected data was used to estimate these models. 
We formulated an optimal assistance-seeking problem for the robot that seeks to optimize team performance while accounting for human trust and secondary task engagement and solved it using Model Predictive Control (MPC). We showed that the optimal assistance-seeking policy is to never seek assistance in low-complexity object-collection trials and to seek assistance in high-complexity object-collection trials only when the trust is below a secondary-task-performance-dependent threshold. Specifically, this threshold decreases with the performance. 
We compared this proposed policy with a greedy baseline policy and conducted experiments with human participants to evaluate the MPC policy. The results showed that the MPC policy outperformed the greedy policy.

\bibliographystyle{IEEEtran}

\bibliography{main}

\end{document}